 \documentclass[usletter, 10 pt, conference]{ieeeconf} 

\IEEEoverridecommandlockouts

\usepackage{xcolor}
\usepackage{xspace}
\usepackage{tcolorbox}

\usepackage{amsmath}
\usepackage{amsthm}
\usepackage{amsfonts}

\usepackage[ruled,vlined,linesnumbered]{algorithm2e}
\PassOptionsToPackage{hyphens}{url}\usepackage[hidelinks]{hyperref}
\usepackage[nameinlink]{cleveref}

\usepackage{enumerate}

\crefname{figure}{Figure}{Figures}
\crefname{algorithm}{Algorithm}{Algorithms}
\newcommand{\highlight}[1]{\textcolor{red}{#1}}

\newcommand{\defproblemobj}[3]
{
	\begin{tcolorbox}[colback=gray!5!white,colframe=gray!75!black]
		\vspace{-1.5mm}
			\begin{tabular*}{\textwidth}{@{\extracolsep{\fill}}lr} #1   \\ \end{tabular*}
			{\bf{Input:}} #2  \\
			{\bf{Objective:}} #3
		\end{tcolorbox}
}
\hypersetup{
	colorlinks = true,
	linkcolor={blue},
	citecolor={blue},
	urlcolor={blue},
}

\newtheoremstyle{mystyle1}
{\topsep}
{\topsep}
{\upshape}
{}
{\bfseries}
{.}
{ }
{\thmname{#1}\thmnumber{ #2}\thmnote{ (#3)}}%

\newtheoremstyle{mystyle2}
{\topsep}
{\topsep}
{\upshape}
{}
{\itshape}
{.}
{ }
{\thmname{#1}\thmnumber{ #2}\thmnote{ (#3)}}%

\theoremstyle{mystyle1}
\newtheorem{theorem}{Theorem}
\newtheorem{lemma}[theorem]{Lemma}
\newtheorem{corollary}[theorem]{Corollary}
\newtheorem{obs}[theorem]{Observation}

\theoremstyle{mystyle2}

\Crefname{algo}{Algorithm}{Algorithms}
\Crefname{ques}{Question}{Questions}
\Crefname{case}{Case}{Cases}
\Crefname{subsection}{Subsection}{Subsections}
\Crefname{obs}{Observation}{Observations}
\Crefname{line}{Line}{Lines}

\newcommand{\N}{\mathbb{N}}
\newcommand{\R}{\mathbb{R}}
\renewcommand{\Im}[1]{\text{Im}\left(#1\right)}
\renewcommand{\Pr}[1]{\textup{Pr}\left(#1\right)}

\newcommand{\order}[1]{\mathcal{O}\left(#1\right)}

\newcommand{\NPH}{\textsc{NP-Hard}\xspace}

\newcommand{\argmax}{\textup{\text{argmax}}\xspace}

\newcommand{\p}[3][]{p^{(#3)}_{#2}#1}
\newcommand{\bfalpha}{\boldsymbol{\alpha}}
\newcommand{\bfbeta}{\boldsymbol{\beta}}
\newcommand{\bfc}{\boldsymbol{c}}

\newcommand{\ImpSea}{\textup{\textsc{Imperfect Searcher}}\xspace}
\newcommand{\OneDImpSea}{\textup{\textsc{1D ImpSea}}\xspace}
\newcommand{\OrdImpSea}{\textup{\textsc{Ordered ImpSea}}\xspace}
\newcommand{\OP}{\textup{\textsc{Orienteering}}\xspace}
\newcommand{\TSP}{\textup{\textsc{Travelling Salesman}}\xspace}

\newcommand{\UKP}{\textup{\textsc{Unbounded Knapsack}}\xspace}
\newcommand{\KP}{\textup{\textsc{Knapsack}}\xspace}

\newcommand{\HP}{\textup{\textsc{Hamiltonian Path}}\xspace}
\newcommand{\TwoDImpSea}{\textup{\textsc{Uniform ImpSea}}\xspace}
\newcommand{\GenImpSea}{\textup{\textsc{General Imperfect Searcher}}\xspace}

\newcommand{\ImpSeaInstance}{(V,P_X,\bfalpha,\bfbeta,\bfc,T)}
\newcommand{\OneDImpSeaInstance}{(V,P_X,\bfbeta,\bfc,T)}
\newcommand{\TwoDImpSeaInstance}{(V,\beta,c,T)}
\newcommand{\KPinstance}{(n, p, w, B)}
\newcommand{\OrdImpSeaInstance}{(V,P_X,\bfbeta,\bfc,T)}
\newcommand{\GenImpSeaInstance}{(V,P_X,\bfbeta,\bfc,T)} 

\title{\LARGE \bf
Provable Methods for Searching with an Imperfect Sensor 
}

\author{Nilanjan Chakraborty$^1$, {Prahlad Narasimhan} Kasthurirangan$^{1,2}$, Joseph S.B. Mitchell$^1$, \\ Linh Nguyen$^1$, and Michael Perk$^3$ 
\thanks{$^1$Stony Brook University, New York, USA}
\thanks{$^2$Email: \href{mailto:prahladnarasim.kasthurirangan@stonybrook.edu}{prahladnarasim.kasthurirangan@stonybrook.edu}}
\thanks{$^3$TU Braunschweig, Lower Saxony, DE}}

\begin{document}

\maketitle

\begin{abstract}
Assume that a target is known to be present at an unknown point among a finite set of locations in the plane. We search for it using a mobile robot that has imperfect sensing capabilities. It takes time for the robot to move between locations and search a location; we have a total time budget within which to conduct the search. We study the problem of computing a search path/strategy for the robot that maximizes the probability of detection of the target. 
Considering non-uniform travel times between points (e.g., based on the distance between them) is crucial for search and rescue applications; such problems have been investigated to a limited extent due to their inherent complexity. In this paper, we describe fast algorithms with performance guarantees for this search problem and some variants, complement them with complexity results, and perform experiments to observe their performance. 
\end{abstract}

\section{Introduction}\label{intro}

A fundamental problem of interest in search and rescue (SAR) is the following: given a mobile robot with imperfect sensing capabilities and a time (or fuel) budget, execute a search of a set of points to find a stationary target. We refer to this problem as \ImpSea (defined formally in \Cref{problem formulation}). With the significant growth in the development and availability of robotic hardware platforms for their use in SAR operations \cite{murphy_search_2008}, it is important to have a thorough theoretical understanding of the relationship between the geometric structure of search tasks and the perceptual capabilities and uncertainties of the robot. Specifically, it is critical to design fast algorithms with performance guarantees that can be translated into practice for search problems set up with realistic modeling assumptions.

Solving (different variants of) \ImpSea has three main components (illustrated in \Cref{flowchart}): we must decide on (I) the subset of points to visit; (II) the order in which to visit them; and (III) how to allocate search effort to these points (visited in this order). Each of these steps have implicit complexity -- Step I has flavors of \OP \cite{vansteenwegen_orienteering_2019}; Step II faces $\Omega(n!)$ orderings; and Step III is similar to \UKP \cite{kellerer_knapsack_2004}. Addressing \ImpSea through these steps naturally generates different variants of the problem, which are of independent interest.

\begin{figure}[ht]
    \centering
    \includegraphics[width=\linewidth]{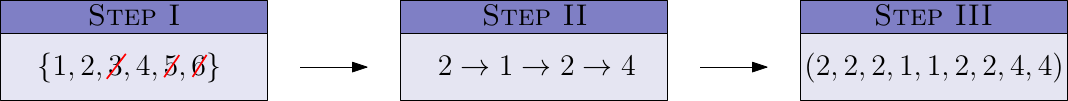}
    \caption{A three-step approach to solve \ImpSea. Say $n = 6$. First, we decide on which on subset to visit; next the order in which to visit them; and finally the number of times to search a point \textit{during} each visit.}
    \label{flowchart}
\end{figure}

The problems of target search (especially those that address one of these three challenges) have been extensively studied by the robotics, operations research, and computational geometry communities. Designing paths with differing notions of optimality is the subject of study of (the many variants of) \textsc{Orienteering} and \TSP: see recent surveys and books on them \cite{gunawan_orienteering_2016,vansteenwegen_orienteering_2019,m-gspno-00,applegate_traveling_2006,crc-2016,saller_systematic_2024}. On the other hand, there is a significant body of work in search theory, 
where optimally allocating search effort has been the study of some of the seminal papers in this field \cite{koopman_theory_1956,koopman_theory_1956-1,koopman_theory_1957,jr_sequential_1967,kadane_discrete_1968}; see surveys and books \cite{stone_theory_1976,benkoski_survey_1991,hohzaki_search_2016} for more recent work. We discuss closely related work in \Cref{related work}.

\begin{figure}[ht]
    \centering
    \includegraphics[width=\linewidth]{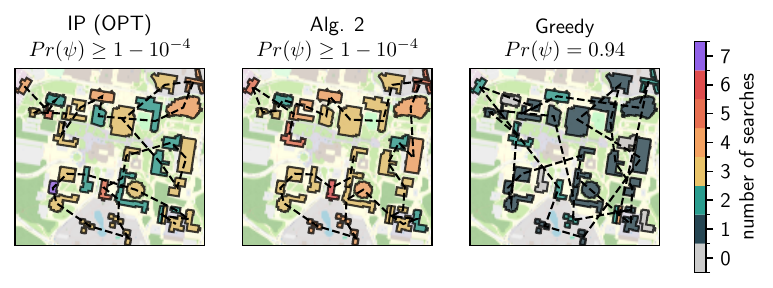}
    \caption{Comparison of 3 algorithms on Stony Brook University buildings. We use Cartopy \cite{Cartopy} and OpenStreetMap \cite{OpenStreetMap} to obtain satellite data. Search costs are proportional to a building's footprint; time budget is a factor of the map's diameter. How do we find a target on one of these rooftops using an aerial vehicle with imperfect sensing capabilities?}
    \label{example of solvers with OSM}
\end{figure}

\textbf{Main Contributions.} We tackle \ImpSea by breaking it into three steps as explained above. Step II is the easiest to overcome: in \Cref{no revisits}, we prove that given a subset of points $V'$ to visit, the optimal schedule visits them in the order given by \textit{any} optimal TSP path on $V'$. Next, in \Cref{1D problem}, we use dynamic programming (DP) to design a psuedopolynomial algorithm to solve Step III -- indeed, we show that this reduces to solving \ImpSea on a line. We complement this by proving our problem, even in 1D, is \NPH. This, therefore, leaves Step I: how do we choose an optimal subset of points to visit? In \Cref{General problem}, we show that if we have uniform priors, then, we can use $k$-TSP paths to compute this subset. We design, in \Cref{Ordered problem}, a DP to compute the optimal subset to visit (and simultaneously the optimal search allocation) when an input points $\{v_1, v_2 \dots v_n\}$ must be searched in the given order. We use this DP to design a heuristic for the general problem in \Cref{Simulation Results}, where we compare it to an (exact) integer linear program and a greedy heuristic (see \Cref{example of solvers with OSM}). 

\section{Problem Formulation}\label{problem formulation}
Consider a finite set, $V= \{v_1, v_2 \dots v_n\} \subset \R^2$, of $n$ points in the plane. A random variable $X\in\{1,2,\ldots,n\}$ denotes the position of a single, stationary (immobile) \textit{target} at $v_X\in V$. Let $p_i = \Pr{X = i}$ be the prior probability (belief estimate) that our target is at $v_i$ and let $P_X = (p_i)^n_{i=1}$ be the discrete probability mass function vector. We assume that the target indeed exists within $V$, so that $\sum^n_{i=1}p_i = 1$; this assumption can be relaxed by adding a point $v_0$ to $V$ that designates the non-presence of the target (think of $v_0$ as a point ``far away"). Note that the probability that the target is at $v_i$ changes over time as we execute a search; it is only initially given by the prior distribution $P_X$.

Our goal is to design an ``optimal" search plan (we define our notion of optimality shortly) to find the target within a given time budget $T$. Formally, an \textit{$s$-step schedule} $\psi$ is a function -- $\psi \colon \{0,1 \dots s\} \mapsto \{1,2 \dots n\}$ ($\psi(0)$ simply denotes the point at which the robot begins the search -- in some applications, this might be constrained due to the presence of a depot; this is easy to handle). A \textit{schedule} is an $s$-step schedule for some $s \in \N$. For a schedule $\psi$, let $|\psi|$ denote the number of steps that $\psi$ has. Each step of a schedule consists of either a movement step, during which the searcher moves between two points of $V$, or a search step, during which the searcher conducts a search at its current location point. It is sometimes more convenient to think of a schedule as a (finite) sequence instead; we switch between these two notions as required. We assume that the searcher moves between points at unit speed; thus, the time required for the searcher to make a movement step from $v_i$ to $v_j$ is $d(v_i,v_j)$, the Euclidean distance between the points. The time required to execute the search of $v_i$ is $c_i$, a natural number (real numbers are easy to handle too). Let $\bfc = (c_i)^n_{i=1}$. The \textit{weight} of step $t > 0$ of $\psi$, denoted by $w_t(\psi)$, is the time taken to perform that step:
\[
    w_t(\psi) = \begin{cases}
        c_{\psi(t)} & \text{ if }\psi(t-1) = \psi(t)\\
        d(v_{\psi(t-1)}, v_{\psi(t)}) & \text{ if }\psi(t-1) \neq \psi(t)
    \end{cases}
\]
The \textit{weight} of $\psi$, denoted by $w(\psi)$, is $\sum^{|\psi|}_{t=1}w_t(\psi)$. This is the total time required to execute $\psi$. We are, therefore, precisely looking for a schedule $\psi$ with $w(\psi) \leq T$.

Our searcher is imperfect -- specifically, if the target \textit{\textbf{is present}} at $v_i$, it reports that the target is \textit{\textbf{not present}} at $v_i$ while searching it with probability $\beta_i$, the false-negative probability. Similarly, $\alpha_i$ denotes the false-positive probability. 
Let $\bfalpha = (\alpha_i)^n_{i=1}$, $\bfbeta = (\beta_i)^n_{i=1}$; we call this searcher an $(\bfalpha,\bfbeta)$--\textit{imperfect searcher}. In this paper, we assume that $\alpha_i=0$, for all $i$, and consider a $(\mathbf{0}$,$\bfbeta)$-imperfect searcher.

Let $Y_{\psi(t)}$ be a Bernoulli random variable that indicates whether or not the searcher reports the target to be present: $Y_{\psi(t)}=1$ (resp., $Y_{\psi(t)}=0$) if the searcher reports the target to be present (resp., not to be present) at step $t$ of the schedule $\psi$, having just searched (at step $t$) the point $\psi(t)$. Then,
\begin{align}
    \Pr{Y_{\psi(t)}= 0 \mid X = \psi(t)} &= \beta_{\psi(t)} \label{no | present}\\
    \Pr{Y_{\psi(t)}= 0 \mid X \neq \psi(t)} &= 1 - \alpha_{\psi(t)} = 1\label{no | not present}\\
    \Pr{Y_{\psi(t)}= 1 \mid X = \psi(t)} &= 1 -\beta_{\psi(t)}& \label{yes | present}\\
    \Pr{Y_{\psi(t)}= 1 \mid X \neq \psi(t)} &= \alpha_{\psi(t)}= 0\label{yes | not present}
\end{align}
There are two fundamental measures of optimality in search problems: minimizing the expected time to detection (ETTD) and maximizing the probability of detection \cite{stone_theory_1976}; we restrict ourselves to the latter in this paper. We use $\Pr{\psi}$ to denote the probability that the target is detected when using the schedule $\psi$ (see \Cref{no revisits} for a rigorous definition). Formally, our problem of interest is:
\defproblemobj{\ImpSea}{A finite set of points $V \subset \R^2$, a target random variable $X$ with a (prior) probability mass function $P_X$ on $V$, a search cost vector $\bfc$, an $(\mathbf{0},\bfbeta)$--imperfect searcher, and a time budget $T$.}{Find a schedule $\psi$ that maximizes $\Pr{\psi}$ subject to $w(\psi) \leq T$.}

\section{Related Work}\label{related work}
Searching for lost targets has been the subject of study for over 70 years; we refer the reader to books and surveys \cite{stone_theory_1976,benkoski_survey_1991,hohzaki_search_2016} cited in \Cref{intro} for a general overview of the field. In this section, we review closely related work and discuss how our results complement existing literature. The problems we discuss in this paper have three main features:
\begin{enumerate}[(i)]
    \item The search space consists of a finite number of points in the Euclidean plane. There is a time or fuel penalty (equal to the distance between the points) for switching between searching different points. This is sometimes referred to as \textit{non-uniform switching costs} in prior literature. In the problems discussed in \Cref{1D problem,Ordered problem}, we also have constraints on potential search paths that the searcher can employ.
    \item There is a (possibly non-uniform) cost to search each point and we are allowed to visit a point without searching it.  
    \item Our searcher is imperfect: there is a non-zero probability that the searcher reports \textsc{no} on searching a point that contains the target. 
\end{enumerate}

\subsection{Closely Related Work}
\ImpSea is clearly a well-motivated fundamental problem in optimal search theory; it is no surprise that it has been the studied for over four decades. To the best of our knowledge, \ImpSea was first described by Lössner and Wegener in 1982 \cite{lossner_discrete_1982}. They devise necessary and sufficient conditions for the existence of an optimal schedule (to minimize ETTD) of a specific type when searching points on the plane using an imperfect searcher. Unfortunately, their (exact) algorithm is impractical since its runs in doubly exponential time (as they explain in the final paragraph of Section 4 in their paper). This was somewhat justified in a later paper written by one of the authors \cite{wegener_optimal_1985}, which shows that this problem is \NPH (even with constant overlook probabilities and search costs). Around the same time, Trummel and Weisinger \cite{trummel_complexity_1986} showed that computing a schedule of a searcher that maximizes the probability of detection is \NPH in (even unweighted) graphs using a reduction from \HP. 

The search for fast practical algorithms for (various versions of) \ImpSea persists with the advancement of various heuristics \cite{chung_decision-making_2007,chung_multi-agent_2008,kress_optimal_2008,chung_analysis_2012,berger_exact_2013,yu_bayesian-based_2020,cheng_scheduling_2021}. Apart from developing these heuristics, these papers also provide a robust mathematical formulation of the problem; indeed we borrow most of our notation from them, especially \cite{chung_analysis_2012}. Crucially, these formulations discretise the search environment into a grid and restrict searcher motion to adjacent cells (with unit switch cost). This method is useful when the search region is continuous (for example, a lost hiker in the woods). However, it induces enormous computational overhead when the search region is discontinuous (for example, searching for those stranded on rooftops during natural disasters) since they depend on the \textit{diameter} of $V$ (the maximum distance between any two points in $V$). We overcome this challenge by bypassing the use of grids altogether and using the underlying geometry instead.

\subsection{Related Work without Feature (i)}

Initial papers on search theory considered problems with no constraints on the search path and no switch costs \cite{koopman_theory_1956,koopman_theory_1956-1,koopman_theory_1957,jr_sequential_1967,kadane_discrete_1968}. Heuristics for problems with non-uniform search costs (with both errors) and no switch costs have been studied in \cite{kriheli_optimal_2013,cheng_scheduling_2021}. Recently, machine learning techniques have also been employed for problems of this type \cite{laperriere-robillard_supervised_2022}. In a paper that discussed a search problem with multiple searchers and no switch costs \cite{song_discrete_2004}, the authors state (in Section 4 of their paper): \textit{``The inclusion of switching delays drastically changes the nature of search problems. The resulting problems are considerably more difficult than the corresponding ones without switching delays"}. 

\subsection{Related Work without Feature (iii)} \label{perfect searcher lit review}
If our searcher is indeed perfect (i.e., $\beta_i = 0$ for all $i$), then \ImpSea reduces to the well known \textsc{Travelling Repairman} (and more generally \textsc{Graph Search}) \cite{koutsoupias_searching_1996,ausiello_salesmen_2000,lau_optimal_2005,van_ee_approximation_2018} when the objective is to minimize ETTD. When we look to maximize the detection probability, it reduces to \OP \cite{golden_orienteering_1987,gunawan_orienteering_2016,vansteenwegen_orienteering_2019}. However, there are no known formulations of \textsc{Orienteering} that accurately model search problems with imperfect searchers (with the latter objective): see Section II B of \cite{mohamed_person_2020} for a review of its relevant variants in the context of search theory.

\ImpSea, therefore, is an important problem to strengthen the algorithmic foundations of: not only because of its relevance in practical SAR operations, but also due to its relation to several other fundamental problems in robotics and operations research. 

\section{Optimal Order of Visits}\label{no revisits}
In this section, we address Step II in our three-step approach to solve \ImpSea: given a subset $V' \subseteq V$ to visit, what is the optimal order to visit them in? We show that this \textit{only} depends on $V'$ (and not also on $P_X$ and $\bfbeta$) -- this is not immediate since $\Pr{\psi}$ clearly also depends on these variables. Our main result (\Cref{optimal no revisit}) in this section is the following: there is an optimal schedule that does not revisit any point (of course, it can search a point several times during a visit). First, we have a simple claim.

\begin{lemma}\label{probability order does not matter main thm}
    $\Pr{\psi} = \sum^n_{i=1}\left((1 - \beta^{b_i}_i) p_i\right)$ for a schedule $\psi$.
\end{lemma}
\begin{proof}
	We have,
	\begin{align*}
		\Pr{\psi} &= \sum^n_{i=1} \Pr{\psi \text{ sees target at }v_i}\\
		& = \sum^n_{i=1} \Pr{\psi \text{ sees target at }v_i \mid X = i} \cdot p_i\\
		&\qquad\quad+ \Pr{\psi \text{ sees target at }v_i \mid X \neq i} \cdot (1 - p_i)\\
		&= \sum^n_{i=1} \left((1 - \Pr{\psi \text{ misses target at }v_i \mid X = i}) \cdot p_i\right)\\
		&\qquad\quad + 0\\
		&=\sum^n_{i=1}\left((1 - \beta^{b_i}_i)\cdot p_i\right)
	\end{align*}
	Note that the second term in the third equality is 0 by \highlight{Equation (4)} while the last equality follows from \highlight{Equation (1)}.
\end{proof}
Note, therefore, that the order in which a schedule visits points has no effect on the probability that it finds the target -- only the number of times it searches a given point (and fails to detect the target) does \cite{song_discrete_2004}. We prove a small observation before the main theorem of this section.

\begin{lemma}\label{lower weight for no revisits}
    Let $\psi$ be a schedule, and let $s_i$ be the number of times that $\psi$ searches $v_i \in V$. Order $\Im{\psi}$ as the nodes appear in the schedule -- say $\Im{\psi} = \{i_1, i_2 \dots i_{s'}\}$. Consider the schedule $\psi'$ as follows: search $v_{i_1}$ $s_{i_1}$ times; then, move to $v_{i_2}$ and search it $s_{i_2}$ times and so on until $v_{i_{s'}}$. Then, $w(\psi') \leq w(\psi)$. 
\end{lemma}
\begin{proof}
    Clearly, the cost for searching nodes (excluding ``travel costs") is the same for both $\psi'$ and $\psi$ -- specifically, this cost is $ \sum^n_{i=1}c_i s_i$. Now, consider an $i_j \in \{i_1, i_2 \dots i_{s'-1}\}$. To get to $v_{i_{j+1}}$ from $v_{i_j}$, $\psi$ walks through a subset of nodes indexed $\{i_1, i_2 \dots i_{j}\}$ (it might search some of them along the way). However, $\psi'$ uses the segment $v_{i_j}v_{i_{j+1}}$. Thus, by \textit{triangle inequality}, $\psi'$ takes a shorter path to $v_{i_{j+1}}$. Since $j$ was arbitrary, $w(\psi') \leq w(\psi)$. 
\end{proof}
Note that our assumption that $V$ is a metric space (specifically a subset of $\R^2$) is crucial for \Cref{lower weight for no revisits}: otherwise, consider a triangle where one of its sides are much larger than the other two. See \Cref{non-metric} for an illustrated example.
\begin{figure}[ht]
	\centering
	\includegraphics[width=0.4\linewidth]{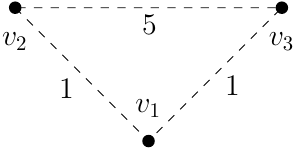}
	\caption{\highlight{Lemma 2} fails if $V$ is not a metric space: schedule $\psi = (1, 1, 2, 1, 3)$ has less weight than $\psi' = (1, 1, 1, 2, 3)$.}
	\label{non-metric}
\end{figure}
\begin{theorem}\label{optimal no revisit}
    Let $\OneDImpSeaInstance$ be an instance of \ImpSea. There exists an optimal schedule $\psi^*$ (say $s = |\psi^*|$) for this instance that does not revisit vertices. That is, for every $v_i \in V$, ${\psi^*}^{-1}(i)$ is an unbroken subsequence of $(0, 1 \dots s)$. 
\end{theorem}
\begin{proof}
    Say $\psi$ is an optimal schedule for this instance of \ImpSea. Note, by definition, that $w(\psi) \leq T$. Let $\psi^*$ be the schedule constructed from $\psi$ as in \Cref{lower weight for no revisits}. Then, $w(\psi^*) \leq w(\psi) \leq T$ (by \Cref{lower weight for no revisits}) and $\Pr{\psi^*} = \Pr{\psi}$ (from \Cref{probability order does not matter main thm}). Thus, $\psi^*$ (which has the property of interest) is also an optimal schedule for $\ImpSeaInstance$. 
\end{proof}

\Cref{optimal no revisit} immediately also solves Step II: it shows that the optimal order to visit a subset $V' \subseteq V$ is to minimize the costs taken to visit all points in $V'$ -- i.e., the order of any optimal TSP path on $V'$.  

\section{Optimal Search Effort Allocation} \label{1D problem}

We now move on to Step III: given a subset of input points to visit and the order to visit them, how do we optimally allocate search effort during a visit to these points? Equivalently, we are given a (parameterized, simple, integrable) curve on which our points of interest lie and we must efficiently search it while moving along this curve. Such constrains are natural in real-world SAR applications too -- see \Cref{southampton}. Searching along a curve immediately maps to searching in 1D (since line integrals are additive); we therefore use \OneDImpSea to refer to this version of \ImpSea.

\defproblemobj{\OneDImpSea}{A finite set of points $V \subset \R$, a target random variable $X$ with a probability mass function $P_X$ on $V$, a search cost vector $\bfc$, an $(\mathbf{0},\bfbeta)$--imperfect searcher, and a time budget $T$.}{Find a schedule $\psi$ that maximizes $\Pr{\psi}$ subject to $w(\psi) \leq T$.}

We use DP to design a pseudopolynomial algorithm for \OneDImpSea (\Cref{DP for 1D ImpSea}) and complement it by showing that it is (weakly) \NPH (\Cref{1D ImpSea is Hard}). Let $\OneDImpSeaInstance$ denote an instance of this problem. We first prove that an optimal schedule only ``moves forward".
\begin{lemma}\label{no backwards}
    There is an optimal schedule $\psi^*$ for an instance $\OneDImpSeaInstance$ of \OneDImpSea such that for all $1\leq t_1 \leq t_2 \leq |\psi^*|$, $\psi^*(t_1) \leq \psi^*(t_2)$.
\end{lemma}

\begin{proof}
    Consider an optimal schedule $\psi$ for this instance of \OneDImpSea. Let $s_i$ denote the number of times that $\psi$ searches $v_i \in V$. Let $v_l$ and $v_r$ be the leftmost and rightmost vertices that $\psi$ visits. Consider the schedule $\psi^*$ that searches each $v_i$ exactly $s_i$ times and visits these vertices in order. Note that the travel costs for $\psi^*$ is exactly $v_r - v_l$ while that for $\psi$ is at least $v_r - v_l$. Since the cost of searching nodes is equal, $w(\psi^*) \leq w(\psi)$. Moreover, as $\Pr{\psi^*} = \Pr{\psi}$ (by \Cref{probability order does not matter main thm}), $\psi^*$ is optimal.
\end{proof}

In other words, \Cref{no backwards} shows that an optimal schedule does not zigzag. We can guess, as our first step, the leftmost and rightmost points of our schedule -- say they are $v_l$ and $v_r$ respectively (run through all $\order{n^2}$ possibilities). This sets up a DP with the flavor of the well known \UKP (see \cite{karp_reducibility_1972, kellerer_knapsack_2004} for an overview of variants of \KP and methods to solve them) where we need to maximize the ``profit" (the probability of detection) constrained by our search budget (which is $T - (v_r - v_l)$). 

\begin{algorithm}[ht]
    \caption{DP for \OneDImpSea given $(l,r)$}
    \label{DP for 1D ImpSea}
    \small
    \KwData{An instance $\OneDImpSeaInstance$ of \OneDImpSea, $l$ and $r$ such that $1 \leq l<r \leq n$}
    \KwResult{The probability of an optimal schedule which starts at $l$ and ends at $r$}
    \tcp{$\tau$ holds the search budget} 
    $\tau \gets \lfloor T - (v_r - v_l) \rfloor$\;
    \tcp{$p[t]$ denotes the optimal probability using budget $t$} 
    $p[t] \gets 0$ for all $0 \leq t \leq \tau$\;
    $i \gets l$\;
    \tcp{A bottom-up construction; in iteration $i$, $p[t]$ considers searching points up to $v_i$}
    \While{$i \leq r$} 
    {
        \tcp{$p'$ holds the probabilities from the previous iteration}
        $p' \gets p$\;
        \tcp{for each $t$, consider searching $v_i$ $j$ times and store the maximum}
        \For{$0 \leq t \leq \tau$}
        {
            $p[t] \gets \max^{\lfloor \frac{t}{c_i} \rfloor}_{j=0} \{p'[t - j \cdot c_i] + (1 - \beta^{j}_i)\cdot p_i\}$ \;\label{DP update 1DImpSea}
        }
        $i \gets i + 1$\;
    }
    \Return{$p[\tau]$}
\end{algorithm}

\begin{theorem}
    \Cref{DP for 1D ImpSea} is correct, runs $\order{nT^2}$-time, and takes $\order{T}$-space.
\end{theorem}
\begin{proof}
    Note that \Cref{DP update 1DImpSea} ``guesses" the number of times, $j$, that we search $v_i$ and uses the remaining budget $t - j \cdot c_i$ to search up to $v_{i-1}$. This indeed produces an optimal schedule by \Cref{optimal no revisit}: the DP considers all schedules that do not revisit vertices. The time and space complexities of the algorithm are readily verified. 
\end{proof}
\begin{corollary}
    \OneDImpSea can be solved in time $\order{n^3 T^2}$ with $\order{T}$ space.
\end{corollary}

\begin{figure}[ht]
    \centering
    \includegraphics[width=\linewidth]{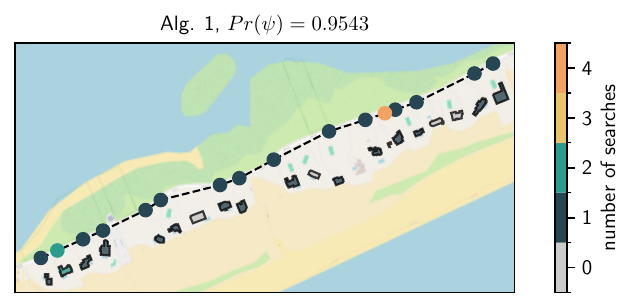} 
    \caption{``Billionare's lane" in the Hamptons in New York, USA \cite{Brennan_2015} (obtained using Cartopy and OSM \cite{Cartopy,OpenStreetMap}). Since there are no other roads, a SAR operation here reduces to \OneDImpSea.}
    \label{southampton}
\end{figure}

We show that \OneDImpSea is weakly \NPH (even when $\bfbeta = \mathbf{0}$) by a reduction from \KP. We note that the hardness proofs given in \cite{wegener_optimal_1985} and \cite{trummel_complexity_1986} do not work for our problem since they reduce from \TSP and \HP respectively, which, of course, are polynomial time solvable in one dimension.
\begin{theorem}
\label{1D ImpSea is Hard}
\label{thm:1dsea_nphard}
    \OneDImpSea is (weakly) \NPH.
\end{theorem}
\begin{proof}
	We show hardness by a reduction from \KP. Consider an instance $\KPinstance$ of \KP: given $n$ items, where $p(i) \in \N$ and $w(i) \in \N$ is the profit and weight of the item $i$, pick the subset of items with maximum total profit subject to its total weight being less than $B$. This is a classic \NPH problem \cite{karp_reducibility_1972,kellerer_knapsack_2004}. We now construct an appropriate instance of \OneDImpSea. Let $v_i = \frac{i}{2n}$ and $c_i = w(i)$ for each $1 \leq i \leq n$; we define $V = \{v_i\}^n_{i=1}$ and $\bfc = (c_i)^n_{i=1}$. Consider a random variable $X$ such that $\Pr{X = i} = \frac{p(i)}{\sum^n_{j=1}p(j)}$. Let $\bfbeta = \mathbf{0}$ -- i.e., the searcher is perfect. Let $T = B + \frac{1}{2}$. Now, consider the instance $\OneDImpSeaInstance$ of \OneDImpSea. Since the searcher is perfect, the probability of a schedule is simply the sum of the probabilities of the points that it visits. 
	
	Consider an optimal schedule $\psi^*$ of this instance. As the searcher is perfect, $\psi^*$ searches the points in its image exactly once. Moreover, we can assume that it travels from left to right (by \highlight{Lemma 4}). Its travel costs, therefore, is at most $\frac{n}{2n} = \frac{1}{2}$ and it has its remaining budget, which is at least $T - \frac{1}{2} = B$, to search. The cost of searching points in $\Im{\psi^*}$ cannot lie in $(B, B + \frac{1}{2}]$ since each $c_i$ is an integer -- hence, $\sum_{i \in \Im{\psi^*}}c_i = \sum_{i \in \Im{\psi^*}}w(i) \leq B$. Since $\psi^*$ maximises the probability of detection, it maximises $\frac{\sum_{i \in \Im{\psi}}p(i)}{\sum^n_{j=1}p(j)}$ and therefore $\sum_{i \in \Im{\psi}}p(i)$. Thus, $\Im{\psi^*}$ is the optimal solution for $\KPinstance$.
\end{proof}

\section{Searching an Ordered Point Set} \label{Ordered problem}
In this section, we consider another variant of \ImpSea, which we call \OrdImpSea. Our motivation is twofold: we prove, for this variant, that Step I (recall \Cref{flowchart}) can be solved efficiently; secondly, in \Cref{Simulation Results}, we show that understanding this variant enables us to design heuristics for \ImpSea. Consider $V = \{v_1, v_2 \dots v_n\}$, where points in $V$ are given in some order. A schedule $\psi$ \textit{respects} this ordering of $V$ if for all $v_i, v_j \in \Im{\psi}$ with $i < j$, there exists no $t_i \in \psi^{-1}(v_i)$ and $t_j \in \psi^{-1}(v_j)$ such that $t_i > t_j$. Note, however, that we do not require a schedule to search all points in $V$; \textit{if} two points are searched, \textit{then} they must be be searched in the given order. Informally, $\psi$ respects the given ordering of $V$ if it never moves ``backward". As in the previous section, we let $\OrdImpSeaInstance$ denote an instance of this problem.

\defproblemobj{\OrdImpSea}{A finite set of ordered points $V \subset \R^2$, a target random variable $X$ with a probability mass function $P_X$ on $V$, a search cost vector $\bfc$, an $(\mathbf{0},\bfbeta)$--imperfect searcher, and a time budget $T$.}{Find a schedule $\psi$ which respects this ordering of $V$ that maximizes $\Pr{\psi}$ subject to $w(\psi) \leq T$.}

We design a DP to solve \OrdImpSea with a similar flavor as \Cref{DP for 1D ImpSea}. Since we are not bound to a curve, we cannot ``guess" the portion of the budget that must be spent travelling. This results in the following challenge: if the distance between points is used to update the remaining budget, then we must discretise distance. The simplest way to do this is to consider a different metric ($l_1$ distance, for instance) or to discretise time. In practice, the latter restriction is not too significant (robots can execute instructions only for multiples of a ``small" measure of time constrained by, for example, clock speeds of onboard systems); prior work in search theory discretise time, see, for example \cite{mohamed_person_2020,morin_ant_2023}. Let $C$ be the number of number of intervals that we divide one unit of time into. We note that this discretisation essentially allows for the searcher to go over budget by an $\varepsilon > 0$; we can choose $C$ carefully ($C \in \order{\varepsilon^{-1}}$) to make $\varepsilon$ as small as required. We explain these details below:

\textit{On discretising time.} Here, we present an alternative to discretising time, which was required in Algorithm 2 for \OrdImpSea. Consider an $\varepsilon > 0$. Given $V \subset \R^2$, on the plane, the total travel cost of an optimal schedule (i.e., the length of the search path) is at most $n\delta$ where $\delta$ is the diameter of the point set $V$. Divide one unit of time into $\frac{n\delta}{\varepsilon}$ parts, and round up all travel costs (i.e., distances between points) to the nearest interval. The total error accumulated by a schedule using this approximation is at most $n\delta \cdot \frac{\varepsilon}{n\delta} = \varepsilon$. Since the cost to search a node is assumed to be an integer, we would not incur any ``approximation penalty" to search a node. Thus, if we are given a leeway of $\epsilon > 0$ to extend the budget, we can use \highlight{Algorithm 2} (with $C = \frac{n \delta}{\varepsilon}$) to find a schedule whose probability of detection is at least that of the optimal schedule (while using budget $T$).

\begin{algorithm}[ht]
    \caption{DP for \OrdImpSea}
    \label{DP for OrdImpSea}
    \small
    \KwData{An instance $\OrdImpSeaInstance$ of \OrdImpSea, $C > 0$}
    \KwResult{The probability of an optimal schedule for $(V,X,\bfbeta,\bfc,T)$, up to an integral multiple of $C$}
    \tcp{scale time according to $C$}
    $\tau \gets \lceil T \cdot C \rceil$\;
    \tcp{$p[i, t]$ denotes the optimal probability using budget $t$ for searching up to $v_i$} 
    $p[0, t] \gets 0$ for all $0 \leq t \leq \tau$\;
    $d(v_i, v_0) \gets 0$ for all $1 \leq i \leq n$\;
    $i \gets 1$\;
    \tcp{A bottom-up construction; at iteration $i$, $p[i,t]$ so that considers searching up to $v_i$ with time $t$}
    \While{$i \leq n$} 
    {
        \tcp{for each $t$, consider searching $v_i$ $j$ times and that the previous point searched was $v_k$}
        \For{$0 \leq t \leq \tau$}
        {
            \tcp{$S$ is the set of possible tuples $(j,k)$ so that remaining time is non-negative}
            $S \gets \{(j,k) \mid j \geq 1 \text{, } 0 \leq k < i \text{, and }t - \lceil(j \cdot c_i + d(v_i, v_k)) \cdot C\rceil \geq 0\}$\;
            $p[i, t] \gets \max \{p[k, t - \lceil(j \cdot c_i + d(v_i, v_k)) \cdot C\rceil] + (1 - \beta^{j}_i)\cdot p_i \mid (j,k) \in S\}$ \; \label{DP Equation for OrdImpSea}
        }
        $i \gets i + 1$\;
    }
    \Return{$\max \{p[i,\tau] \mid 1 \leq i \leq n\}$}
\end{algorithm}

\begin{theorem}
    \Cref{DP for OrdImpSea} solves \OrdImpSea in $\order{n^2T^2C^2}$-time, using $\order{nTC}$-space.
\end{theorem}
\begin{proof}
    The time and space complexity analysis of \Cref{DP for OrdImpSea} is straightforward. As for correctness, we study the assignment in \Cref{DP Equation for OrdImpSea}. To compute the maximum probability of detection using a budget $t$ (scaled using $C$) for searching up to $v_i$ (and $v_i$ is searched), i.e, $p[i, t]$; we ``guess" (i.e., try all possibilities) $j$, the number of times that we search $v_i$ and the previous vertex $v_k$ that the schedule searched (which is why $j \geq 1$). Using the remaining budget $t - \lceil(j \cdot c_i + d(v_i, v_k)) \cdot C \rceil$ (note the scaling using $C$), we search up to $v_k$. $S$ is simply the set of permissible tuples $(j,k)$ so that the remaining time is non-negative. The algorithm ``guesses" the last $v_i$ to be searched and returns $p[i,\tau]$. Observe the implicit use of \Cref{optimal no revisit} as in \Cref{DP for 1D ImpSea}.
\end{proof}
Note that the proof of hardness of \OneDImpSea (\Cref{1D ImpSea is Hard}) immediately shows that \OrdImpSea is also weakly \NPH: it inherently assumes an ordering on $V$ (ascending order on the $x$-axis) and visits the vertices in this order. Thus, we have:
\begin{theorem}
    \OrdImpSea is (weakly) \NPH.
\end{theorem}

\section{Imperfect Searcher with Uniform Priors} \label{General problem}
Another scenario where Step I becomes tractable is when we have uniform priors: we call this problem \TwoDImpSea.
\defproblemobj{\TwoDImpSea}{A finite set of points $V \subset \R^2$, a target random variable $X$ with probability mass function $P_X \sim U$ on $V$, a search cost vector $c \cdot \mathbf{1}$, an $(\mathbf{0},\beta \cdot \mathbf{1})$--imperfect searcher, and a time budget $T$.}{Find a schedule $\psi$ that maximizes $\Pr{\psi}$ subject to $w(\psi) \leq T$.}

We use $\TwoDImpSeaInstance$ to denote an instance of this problem. As the main result in this section (\Cref{dual approx for 2D}), we show that by slightly relaxing the budget constraint (within a factor of $(1 + \varepsilon)$ where $\varepsilon > 0$ can be arbitrarily small) we can achieve at least as much probability of detection as an optimal schedule of the given budget in polynomial time. For any schedule $\psi$, we let $\mathbf{s}(\psi) = \{s_i(\psi)\}^n_{i=1}$ denote the (multi) set of the number of times that each $v_i \in V$ is searched -- i.e, $v_i$ is searched $s_i(\psi)$ times by $\psi$ (we drop the parenthesis when appropriate). Crucially, $\mathbf{s}(\psi)$ is \textit{not} a sequence: its ordering does not matter. We have a straightforward claim.

\begin{obs}\label{same number of searches}
	Consider an instance $\TwoDImpSeaInstance$ of this problem. If $\mathbf{s}(\psi) = \mathbf{s}(\psi')$ of two schedules $\psi$ and $\psi'$, then $\Pr{\psi} = \Pr{\psi'}$. Moreover, the search costs for $\psi$ and $\psi'$ are equal. 
\end{obs}
Note the subtle difference between our comment after Lemma 1 and the (first part of the) above claim: in general, the probabilities of two schedules which search a point the same number of times are the same (i.e, $(s_i(\psi))^n_{i=1} = (s_i(\psi'))^n_{i=1}$); however, since $P_X \sim U$, we just require the (multi) sets $\{s_i(\psi)\}^n_{i=1}$ and $\{s_i(\psi')\}^n_{i=1}$ to be equal! The second part of the claim follows from our assumption that we have uniform search costs. \Cref{same number of searches} \textit{does not} imply that the weights of $\psi$ and $\psi'$ are the same: their travel costs might differ, possibly significantly. 

\begin{theorem}\label{dual approx for 2D}
	Let $\psi^*$ be an optimal schedule of an instance $\TwoDImpSeaInstance$ of \TwoDImpSea. For any fixed $\varepsilon>0$, we can compute, in polynomial time, a schedule $\psi$ such that $w(\psi) \leq (1+\varepsilon)T$  and $\Pr{\psi} \geq \Pr{\psi^*}$.
\end{theorem}

\begin{proof}
	Assume that the optimal schedule $\psi^*$ visits exactly $k$ points in $V$. Without loss in generality, assume that $\Im{\psi^*} = \{v_1, v_2 \dots v_k\}$. Consider a schedule $\psi'$ with $\mathbf{s}(\psi') = \mathbf{s}(\psi^*)$ that travels the shortest route spanning \textit{any} $k$ points in $V$. Then, since the travel costs for $\psi'$ is at most that of $\psi^*$ and the search costs are equal (by \Cref{same number of searches}), $w(\psi') \leq w(\psi^*)$. Moreover, $\Pr{\psi'} = \Pr{\psi^*}$ (also by \Cref{same number of searches}). Thus, it suffices to look for schedules traversing shortest routes between $k$ points: without loss in generality, we assume that $\psi^*$ is such a schedule.  Denote the travel time of $\psi^*$ by $L$, then the search time of $\psi^*$ is $T - L$. Note that $\sum_{i=1}^{k}s_i \leq \lfloor \frac{T- L }{c}\rfloor$ and that $\Pr{\psi^*}= \sum^k_{i=1} \frac{1}{n}\left(1 - \beta^{s_i}\right)$ (a simplification of Lemma 1 when $P_X \sim U$).
	
	We now show that $\max_{i,j}|s_i - s_j| = 1$. For any distinct $i$ and $j$, suppose $s_i + 2 \le s_j$. Let $s_i'$ and $s_j'$ be such that $s_i + s_j = s_i' + s_j'$ and $s_i < s_i' \leq s_j' < s_j$. Then,
	\begin{align*}
		0 &> \underbrace{(\beta^{s_i} - \beta^{s_j'})}_{<0}\underbrace{(\beta^{s_i'-s_i} - 1)}_{<0}\\
		&= \beta^{s'_i} - \beta^{s_i} + \beta^{s'_i - s_i + s_j'} - \beta^{s'_j}\\
		&= \beta^{s_i'} + \beta^{s_j'} - (\beta^{s_i} + \beta^{s_j})
	\end{align*}
	
	The last equality follows as $s_j - s_j' = s_i' - s_i$. Therefore, $ \frac{1}{n}(2 - (\beta^{s_i'} + \beta^{s_j'})) > \frac{1}{n}(2 - (\beta^{s_i} + \beta^{s_j}))$.
	
	Define, for all other $l$, $s'_l = s_l$. Consider $\psi'$ which executes a shortest route spanning $\{v_1, v_2 \dots v_k\}$ and searches those vertices $s'_i$ times. As $s_i + s_j = s_i' + s_j'$, $w(\psi') \leq w(\psi^*) \leq T$. Moreover, by the assertion from the previous paragraph,
	\begin{align*}
		& \sum^k_{l=1} \frac{1}{n}\left(1 - \beta^{s'_l}\right) > \sum^k_{l=1} \frac{1}{n}\left(1 - \beta^{s_l}\right)\\
		\implies & \Pr{\psi'} > \Pr{\psi^*}
	\end{align*}
	This contradicts our assumption that $\psi^*$ is optimal. Thus, we have shown that search effort is allocated (almost) equally in an optimal schedule.
	
	We are ready to describe our (dual) approximation algorithm. Guess the number of points $k$ that the optimal schedule visits: i.e., run through all $k \in \{1, 2 \dots n\}$, and pick the schedule that maximizes probability. Approximate the shortest $k$-TSP tour to a factor of $(1 + \varepsilon)$ using a known polynomial-time approximation scheme (PTAS) (see \cite{arora_polynomial_1998,mitchell_guillotine_1999}). The route returned by the $k$-TSP approximation is no longer than $(1 + \varepsilon)L$ (recall that $L$ is the total travel time by the optimal schedule). Hence, given a budget of time $(1 + \varepsilon)T$, the remaining time for searching is no smaller than $\left\lfloor \frac{(1+ \varepsilon)T- (1 + \varepsilon)L}{c} \right\rfloor$. If we divide that search time roughly equally between points, we get a schedule with probability of detection at least  $\Pr{\psi^*}$.
\end{proof}

Finally, we show that \TwoDImpSea is \NPH. This complements a similar result in \cite{wegener_optimal_1985}, where the authors seek to minimize expected time until detection. Our proof is significantly shorter.

\begin{lemma} 
	\label{lem:2d_nphard}
	\TwoDImpSea is \NPH.
\end{lemma}
\begin{proof}
	We show that even with a perfect searcher and instantaneous search ($\bfbeta = \bfc = \boldsymbol{0}$), \TwoDImpSea is \NPH, using a reduction from \OP: given $n$ points in the Euclidean plane and a budget of length $T$, find a path whose length does not exceed $T$ that visits as many points as possible. The \OP problem is known to be strongly \NPH \cite{golden_orienteering_1987}. Clearly, if the optimal search route for the instance $(V, 0, 0, T)$ has probability of detection of $\frac{k}{n}$, we have an optimal route visiting $k$ points in the corresponding \OP instance, and vice versa.
\end{proof}
\section{Imperfect Searcher -- IP and Heuristics}\label{Simulation Results}
\newcommand{\orienteeringBenchmark}{\emph{orienteering}}
The experiments were carried out on a regular desktop workstation with an
AMD Ryzen 7 5800X ($8\times {3.8}\,$GHz) CPU and $128\,$GB of RAM.
Code and data are available\footnote{\url{https://gitlab.ibr.cs.tu-bs.de/alg/imperfect-sensor-search}}.
Benchmark instances were generated from instances for \OP~\cite{gunawan2016orienteering} where each point has an associated score and the goal is to maximize total collected score in a given time budget. We scaled the scores from the benchmark sets Tsiligirides~\cite{tsiligirides1984heuristic} and Chao \textit{et al.}\cite{chao1996fast} to obtain $P_X$. $\bfbeta$ was sampled from a Dirichlet distribution. For each instance from the benchmark set we generated $10$ instances with small random search costs; resulting in $890$ total instances.

We implemented three different solvers: (i) an exact integer linear programming solver (IP); (ii) a greedy heuristic; (iii-iv) \Cref{DP for 1D ImpSea,DP for OrdImpSea} with an ordering that stems from an approximate TSP path \cite{christofides_worst-case_2022}. 
All solvers were executed with a time limit of $300$s. \Cref{DP for 1D ImpSea,DP for OrdImpSea} were executed with values $C=\ $1, 10, 20. Designing, implementing, and optimising the IP for \ImpSea is of independent interest; we omit the details here due to lack of space. The greedy heuristic we use is the following: at each time step, select the next point to search based on the greatest ``bang for the buck" (i.e., the ratio of belief probability to the cost of moving to that point and searching it). The details are presented in \Cref{Greedy for GenImpSea}.  

\begin{algorithm}[ht]
	\caption{Greedy Algorithm for \GenImpSea}
	\label{Greedy for GenImpSea}
	\KwData{An instance $\GenImpSeaInstance$ of \GenImpSea}
	\KwResult{The probability of a greedy schedule}
	\tcp{$s[i]$ denotes the number of times that we have searched $v_i$} 
	$s[i] \gets 0$ for all $1 \leq i \leq n$\;
	\tcp{$p[i]$ is the probability that the target is at $v_i$ at the current time step}
	$p \gets P_X$\;
	\tcp{$\tau$ denotes the remaining budget}
	$\tau \gets T$\;
	\tcp{$r$ denotes the position of the robot; choose the best "bang for your buck" to start}
	$r \gets \argmax \{\frac{\beta_ip[i]}{c_i} \mid 1 \leq i \leq n\}$\;
	$s[r] \gets s[r] +1$\;
	$\tau \gets \tau - c_r$\;
	$p \gets \textsc{UpdateProb}(p, r)$\;
	\tcp{Choose the best "bang for your buck" as long as $\tau \geq 0$}
	\While{$\tau \geq 0$} 
	{
		$r' \gets \argmax \{\frac{\beta_ip[i]}{c_i + d(v_r, v_i)}\mid 1 \leq i \leq n\ \land \ c_i + d(v_r, v_i) \leq \tau\}$\;
		\If{$r' \gets \textsc{null}$}
		{
			$\tau \gets 0$\;
		}
		\Else
		{
			$\tau \gets \tau - c_{r'} - d(v_{r}, v_{r'})$\;
			$p \gets \textsc{UpdateProb}(p, r')$\;
			$s[r'] \gets s[r'] +1$\;
			$r \gets r'$\;
		}
	}
	$\Pr{\psi} \gets \textsc{ComputeProb}(\bfbeta, s)$\;
	\Return{$\Pr{\psi}$}
\end{algorithm}

\cref{fig:orienteering-benchmark graph} shows the (absolute) gaps to the best bound found by the IP and the runtimes for all solvers.
IP was able to find an optimal solution for $556$ instances. \cref{DP for OrdImpSea} consistently finds solutions very close ($<5\%$) to the best bound within a solve time that is only a fraction of that of the IP. As expected, larger $C$ values yield better solutions in terms of quality. Despite the greedy algorithm being the fastest solver, it finds solutions far away from the best bound.

We also generated instances using OpenStreetMap \cite{Cartopy,OpenStreetMap} -- see \Cref{southampton,example of solvers with OSM}. Our points of interest are buildings and their footprints equal to search costs. $T$ was generated proportional to the diameter of the map while $P_X$ and $\bfbeta$ were sampled from a Dirichlet distribution.

\begin{figure}
    \centering
    \includegraphics[width=\columnwidth]{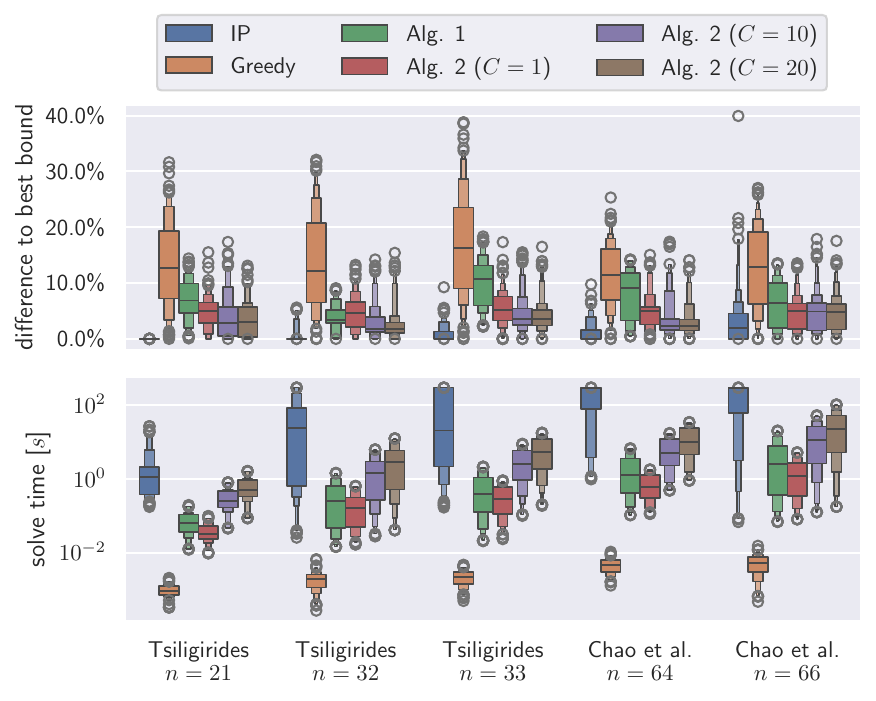}
    \caption{Gaps and runtime of all solvers and configurations on benchmarks. Instances of the same size are grouped together. Note that the $y$-axis in the bottom figure is a logarithmic scale.}
    \label{fig:orienteering-benchmark graph}
\end{figure}

\begin{figure}
    \centering
    \includegraphics[width=\columnwidth]{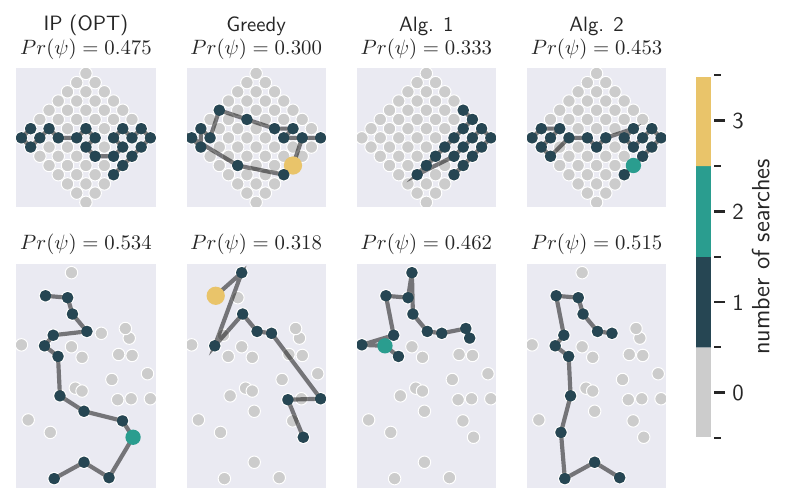}
    \caption{Examples of benchmark instances from \cite{chao1996fast} (top) and \cite{tsiligirides1984heuristic} (bottom).}
    \label{fig:orienteering-benchmark examples}
\end{figure}

\section{Conclusion}\label{conclusion}
In this paper, we defined several variants of \ImpSea, where the objective is to move between points in the plane with a searcher with imperfect sensing capabilities in search of a stationary target. We described a three-step approach (\Cref{flowchart}) to solve this problem. We show that Step II can be handled without computational overhead in \Cref{no revisits}. We show that Step III is equivalent to solving the 1D version of this problem (which itself is \NPH) -- we describe a pseudo-polynomial DP to solve it. When we have additional restrictions on the input parameters, we show that Step I, too, can be solved optimally (in \Cref{Ordered problem}, the input set is ordered; \Cref{General problem} considers uniform priors). The algorithms to solve these versions inform the heuristic we present in \Cref{Simulation Results} for \ImpSea: where we compare it to the IP we design, as well as a greedy heuristic. While results in this paper were stated for points in the plane, we note that they apply, with minimal modifications, to any metric space.  
\newcommand{\longer}[1]{{}}

There are several directions for further research. Instead of maximizing probability of detection, we may seek to minimize the expected time to detection (ETTD), for which our key technical result, \Cref{optimal no revisit}, no longer holds \cite{lossner_discrete_1982}. Other objective functions, and allowing false positives ($\bfalpha \neq \mathbf{0}$) are of interest -- current results are largely limited to heuristics and impose grids on the search area \cite{chung_decision-making_2007,chung_multi-agent_2008,chung_analysis_2012,barkaoui_information-theoretic-based_2014,berger_information_2015,teller_minimizing_2019,barkaoui_evolutionary_2019,mohamed_person_2020,yu_bayesian-based_2020,cheng_scheduling_2021}. 
Searching with swarms \cite{senanayake_search_2016} may include a (possibly heterogeneous) team of searchers and multiple (possibly mobile) targets \cite{wong_multi-vehicle_2005,chung_multi-agent_2008,lo_toward_2012,wettergren_discrete_2014,dell_using_1996,lau_discounted_2008,raap_moving_2019,delavernhe_planning_2021}. We could allow a searcher at point $p$ to search a neighborhood of $p$, with domain-dependent travel costs and having an imperfect searcher with detection probability depending on the distance to $p$.  Progress on this front has considered perfect searchers \cite{dasgupta_aggregation-based_2004,meghjani_multi-target_2016,morin_ant_2023,huynh_et_al:LIPIcs.SWAT.2024.27} or uniform travel costs \cite{sato_path_2010,morin_hybrid_2013,banerjee_multi-agent_2023,collins_probabilistically_2024}.
\longer{With the increased availability of swarm robots \cite{senanayake_search_2016}, we could also consider a (possibly heterogeneous) team of searchers and multiple (possibly mobile) targets \cite{wong_multi-vehicle_2005,chung_multi-agent_2008,lo_toward_2012,wettergren_discrete_2014,dell_using_1996,lau_discounted_2008,raap_moving_2019,delavernhe_planning_2021}. Finally, we discuss our assumption that a searcher can only search the point that they occupy -- this works well when we are planning on a macro level (e.g., a building roof among others). But how do we search for a target in the continuous setting (i.e, how do we search the roof \textit{itself})? A reasonable model is to say that the searcher has a probability (which might depend on the distance to the target) to see the target if it is within a region around it. Progress on this front has considered perfect searchers \cite{dasgupta_aggregation-based_2004,meghjani_multi-target_2016,morin_ant_2023,huynh_et_al:LIPIcs.SWAT.2024.27} or uniform travel costs \cite{sato_path_2010,morin_hybrid_2013,banerjee_multi-agent_2023,collins_probabilistically_2024}; we believe that a more extensive study of imperfect sensing in the continuous setting with travel costs proportional to the distance traveled by the searcher is crucial.}   

\bibliographystyle{IEEEtran}
\bibliography{refs}

\begin{thebibliography}{10}
\providecommand{\url}[1]{#1}
\csname url@samestyle\endcsname
\providecommand{\newblock}{\relax}
\providecommand{\bibinfo}[2]{#2}
\providecommand{\BIBentrySTDinterwordspacing}{\spaceskip=0pt\relax}
\providecommand{\BIBentryALTinterwordstretchfactor}{4}
\providecommand{\BIBentryALTinterwordspacing}{\spaceskip=\fontdimen2\font plus
\BIBentryALTinterwordstretchfactor\fontdimen3\font minus
  \fontdimen4\font\relax}
\providecommand{\BIBforeignlanguage}[2]{{%
\expandafter\ifx\csname l@#1\endcsname\relax
\typeout{** WARNING: IEEEtran.bst: No hyphenation pattern has been}%
\typeout{** loaded for the language `#1'. Using the pattern for}%
\typeout{** the default language instead.}%
\else
\language=\csname l@#1\endcsname
\fi
#2}}
\providecommand{\BIBdecl}{\relax}
\BIBdecl

\bibitem{murphy_search_2008}
\BIBentryALTinterwordspacing
R.~R. Murphy, S.~Tadokoro, D.~Nardi, A.~Jacoff, P.~Fiorini, H.~Choset, and
  A.~M. Erkmen, ``\BIBforeignlanguage{en}{Search and {Rescue} {Robotics}},'' in
  \emph{\BIBforeignlanguage{en}{Springer {Handbook} of {Robotics}}},
  B.~Siciliano and O.~Khatib, Eds.\hskip 1em plus 0.5em minus 0.4em\relax
  Berlin, Heidelberg: Springer, 2008, pp. 1151--1173. [Online]. Available:
  \url{https://doi.org/10.1007/978-3-540-30301-5\_51}
\BIBentrySTDinterwordspacing

\bibitem{vansteenwegen_orienteering_2019}
\BIBentryALTinterwordspacing
P.~Vansteenwegen and A.~Gunawan, \emph{\BIBforeignlanguage{en}{Orienteering
  {Problems}: {Models} and {Algorithms} for {Vehicle} {Routing} {Problems} with
  {Profits}}}, ser. {EURO} {Advanced} {Tutorials} on {Operational}
  {Research}.\hskip 1em plus 0.5em minus 0.4em\relax Cham: Springer
  International Publishing, 2019. [Online]. Available:
  \url{http://link.springer.com/10.1007/978-3-030-29746-6}
\BIBentrySTDinterwordspacing

\bibitem{kellerer_knapsack_2004}
\BIBentryALTinterwordspacing
H.~Kellerer, U.~Pferschy, and D.~Pisinger,
  \emph{\BIBforeignlanguage{en}{Knapsack {Problems}}}.\hskip 1em plus 0.5em
  minus 0.4em\relax Berlin, Heidelberg: Springer, 2004. [Online]. Available:
  \url{http://link.springer.com/10.1007/978-3-540-24777-7}
\BIBentrySTDinterwordspacing

\bibitem{gunawan_orienteering_2016}
\BIBentryALTinterwordspacing
A.~Gunawan, H.~C. Lau, and P.~Vansteenwegen,
  ``\BIBforeignlanguage{en}{Orienteering {Problem}: {A} survey of recent
  variants, solution approaches and applications},''
  \emph{\BIBforeignlanguage{en}{European Journal of Operational Research}},
  vol. 255, no.~2, pp. 315--332, Dec. 2016. [Online]. Available:
  \url{https://www.sciencedirect.com/science/article/pii/S037722171630296X}
\BIBentrySTDinterwordspacing

\bibitem{m-gspno-00}
J.~S.~B. Mitchell, ``Geometric shortest paths and network optimization,'' in
  \emph{Handbook of Computational Geometry}, J.-R. Sack and J.~Urrutia,
  Eds.\hskip 1em plus 0.5em minus 0.4em\relax Amsterdam: Elsevier Science
  Publishers B.V. North-Holland, 2000, pp. 633--701.

\bibitem{applegate_traveling_2006}
\BIBentryALTinterwordspacing
D.~L. Applegate, R.~E. Bixby, V.~Chvatál, and W.~J. Cook, \emph{The
  {Traveling} {Salesman} {Problem}: {A} {Computational} {Study}}.\hskip 1em
  plus 0.5em minus 0.4em\relax Princeton University Press, 2006. [Online].
  Available: \url{https://www.jstor.org/stable/j.ctt7s8xg}
\BIBentrySTDinterwordspacing

\bibitem{crc-2016}
J.~S.~B. Mitchell, ``Shortest paths and networks,'' in \emph{Handbook of
  Discrete and Computational Geometry (3rd Edition)}, J.~E.~G. Csaba~T\'oth,
  Joseph~O'Rourke, Ed.\hskip 1em plus 0.5em minus 0.4em\relax Boca Raton, FL:
  Chapman \& Hall/CRC, 2017, ch.~31, pp. 811--848.

\bibitem{saller_systematic_2024}
\BIBentryALTinterwordspacing
S.~Saller, J.~Koehler, and A.~Karrenbauer, ``A {Systematic} {Review} of
  {Approximability} {Results} for {Traveling} {Salesman} {Problems} leveraging
  the {TSP}-{T3CO} {Definition} {Scheme},'' Jan. 2024, arXiv:2311.00604 [cs].
  [Online]. Available: \url{http://arxiv.org/abs/2311.00604}
\BIBentrySTDinterwordspacing

\bibitem{koopman_theory_1956}
\BIBentryALTinterwordspacing
B.~O. Koopman, ``The {Theory} of {Search}. {I}. {Kinematic} {Bases},''
  \emph{Operations Research}, vol.~4, no.~3, pp. 324--346, Jun. 1956,
  publisher: INFORMS. [Online]. Available:
  \url{https://pubsonline.informs.org/doi/abs/10.1287/opre.4.3.324}
\BIBentrySTDinterwordspacing

\bibitem{koopman_theory_1956-1}
\BIBentryALTinterwordspacing
------, ``The {Theory} of {Search}. {II}. {Target} {Detection},''
  \emph{Operations Research}, vol.~4, no.~5, pp. 503--531, Oct. 1956,
  publisher: INFORMS. [Online]. Available:
  \url{https://pubsonline.informs.org/doi/10.1287/opre.4.5.503}
\BIBentrySTDinterwordspacing

\bibitem{koopman_theory_1957}
\BIBentryALTinterwordspacing
------, ``The {Theory} of {Search},'' \emph{Operations Research}, vol.~5,
  no.~5, pp. 613--626, Oct. 1957, publisher: INFORMS. [Online]. Available:
  \url{https://pubsonline.informs.org/doi/abs/10.1287/opre.5.5.613}
\BIBentrySTDinterwordspacing

\bibitem{jr_sequential_1967}
\BIBentryALTinterwordspacing
M.~C.~C. Jr, ``A {Sequential} {Search} {Procedure},'' \emph{The Annals of
  Mathematical Statistics}, vol.~38, no.~2, pp. 494--502, Apr. 1967, publisher:
  Institute of Mathematical Statistics. [Online]. Available:
  \url{https://projecteuclid.org/journals/annals-of-mathematical-statistics/volume-38/issue-2/A-Sequential-Search-Procedure/10.1214/aoms/1177698965.full}
\BIBentrySTDinterwordspacing

\bibitem{kadane_discrete_1968}
\BIBentryALTinterwordspacing
J.~B. Kadane, ``Discrete search and the {Neyman}-{Pearson} {Lemma},''
  \emph{Journal of Mathematical Analysis and Applications}, vol.~22, no.~1, pp.
  156--171, Apr. 1968. [Online]. Available:
  \url{https://www.sciencedirect.com/science/article/pii/0022247X68901674}
\BIBentrySTDinterwordspacing

\bibitem{stone_theory_1976}
L.~Stone, \emph{\BIBforeignlanguage{en}{Theory of {Optimal} {Search}}}.\hskip
  1em plus 0.5em minus 0.4em\relax Elsevier, Jan. 1976, google-Books-ID:
  DFLpiYM9cg8C.

\bibitem{benkoski_survey_1991}
\BIBentryALTinterwordspacing
S.~J. Benkoski, M.~G. Monticino, and J.~R. Weisinger,
  ``\BIBforeignlanguage{en}{A survey of the search theory literature},''
  \emph{\BIBforeignlanguage{en}{Naval Research Logistics (NRL)}}, vol.~38,
  no.~4, pp. 469--494, 1991, \_eprint:
  https://onlinelibrary.wiley.com/doi/pdf/10.1002/1520-6750\%28199108\%2938\%3A4\%3C469\%3A\%3AAID-NAV3220380404\%3E3.0.CO\%3B2-E.
  [Online]. Available:
  \url{https://onlinelibrary.wiley.com/doi/abs/10.1002/1520-6750%28199108%2938%3A4%3C469%3A%3AAID-NAV3220380404%3E3.0.CO%3B2-E}
\BIBentrySTDinterwordspacing

\bibitem{hohzaki_search_2016}
R.~Hohzaki, ``Search {Games}: {Literature} and {Survey},'' \emph{Journal of the
  Operations Research Society of Japan}, vol.~59, no.~1, pp. 1--34, 2016.

\bibitem{Cartopy}
\BIBentryALTinterwordspacing
{Met Office}, \emph{Cartopy: a cartographic python library with a Matplotlib
  interface}, Exeter, Devon, 2010 - 2015. [Online]. Available:
  \url{https://scitools.org.uk/cartopy}
\BIBentrySTDinterwordspacing

\bibitem{OpenStreetMap}
{OpenStreetMap contributors}, ``{Planet dump retrieved from
  https://planet.osm.org },'' \url{ https://www.openstreetmap.org }, 2017.

\bibitem{lossner_discrete_1982}
\BIBentryALTinterwordspacing
U.~Lössner and I.~Wegener, ``Discrete {Sequential} {Search} with {Positive}
  {Switch} {Cost},'' \emph{Mathematics of Operations Research}, vol.~7, no.~3,
  pp. 426--440, Aug. 1982, publisher: INFORMS. [Online]. Available:
  \url{https://pubsonline.informs.org/doi/abs/10.1287/moor.7.3.426}
\BIBentrySTDinterwordspacing

\bibitem{wegener_optimal_1985}
\BIBentryALTinterwordspacing
I.~Wegener, ``Optimal search with positive switch cost is {NP}-hard,''
  \emph{Information Processing Letters}, vol.~21, no.~1, pp. 49--52, Jul. 1985.
  [Online]. Available:
  \url{https://www.sciencedirect.com/science/article/pii/0020019085901085}
\BIBentrySTDinterwordspacing

\bibitem{trummel_complexity_1986}
\BIBentryALTinterwordspacing
K.~E. Trummel and J.~R. Weisinger, ``The {Complexity} of the {Optimal}
  {Searcher} {Path} {Problem},'' \emph{Operations Research}, vol.~34, no.~2,
  pp. 324--327, 1986, publisher: INFORMS. [Online]. Available:
  \url{https://www.jstor.org/stable/170828}
\BIBentrySTDinterwordspacing

\bibitem{chung_decision-making_2007}
T.~H. Chung and J.~W. Burdick, ``A {Decision}-{Making} {Framework} for
  {Control} {Strategies} in {Probabilistic} {Search},'' in \emph{Proceedings
  2007 {IEEE} {International} {Conference} on {Robotics} and {Automation}},
  Apr. 2007, pp. 4386--4393, iSSN: 1050-4729.

\bibitem{chung_multi-agent_2008}
------, ``Multi-agent probabilistic search in a sequential decision-theoretic
  framework,'' in \emph{2008 {IEEE} {International} {Conference} on {Robotics}
  and {Automation}}, May 2008, pp. 146--151, iSSN: 1050-4729.

\bibitem{kress_optimal_2008}
\BIBentryALTinterwordspacing
M.~Kress, K.~Y. Lin, and R.~Szechtman, ``\BIBforeignlanguage{en}{Optimal
  discrete search with imperfect specificity},''
  \emph{\BIBforeignlanguage{en}{Mathematical Methods of Operations Research}},
  vol.~68, no.~3, pp. 539--549, Dec. 2008. [Online]. Available:
  \url{https://doi.org/10.1007/s00186-007-0197-2}
\BIBentrySTDinterwordspacing

\bibitem{chung_analysis_2012}
T.~H. Chung and J.~W. Burdick, ``Analysis of {Search} {Decision} {Making}
  {Using} {Probabilistic} {Search} {Strategies},'' \emph{IEEE Transactions on
  Robotics}, vol.~28, no.~1, pp. 132--144, Feb. 2012, conference Name: IEEE
  Transactions on Robotics.

\bibitem{berger_exact_2013}
\BIBentryALTinterwordspacing
J.~Berger, N.~Lo, and M.~Noel, ``Exact {Solution} for {Search}-and-{Rescue}
  {Path} {Planning},'' \emph{International Journal of Computer and
  Communication Engineering}, pp. 266--271, 2013. [Online]. Available:
  \url{http://www.ijcce.org/index.php?m=content&c=index&a=show&catid=31&id=201}
\BIBentrySTDinterwordspacing

\bibitem{yu_bayesian-based_2020}
\BIBentryALTinterwordspacing
L.~Yu and D.~Lin, ``\BIBforeignlanguage{en}{Bayesian-{Based} {Search}
  {Decision} {Framework} and {Search} {Strategy} {Analysis} in {Probabilistic}
  {Search}},'' \emph{\BIBforeignlanguage{en}{Scientific Programming}}, vol.
  2020, p. e8865381, Nov. 2020, publisher: Hindawi. [Online]. Available:
  \url{https://www.hindawi.com/journals/sp/2020/8865381/}
\BIBentrySTDinterwordspacing

\bibitem{cheng_scheduling_2021}
\BIBentryALTinterwordspacing
T.~C.~E. Cheng, B.~Kriheli, E.~Levner, and C.~T. Ng,
  ``\BIBforeignlanguage{en}{Scheduling an autonomous robot searching for hidden
  targets},'' \emph{\BIBforeignlanguage{en}{Annals of Operations Research}},
  vol. 298, no.~1, pp. 95--109, Mar. 2021. [Online]. Available:
  \url{https://doi.org/10.1007/s10479-019-03141-1}
\BIBentrySTDinterwordspacing

\bibitem{kriheli_optimal_2013}
\BIBentryALTinterwordspacing
B.~Kriheli and E.~Levner, ``Optimal {Search} and {Detection} of {Clustered}
  {Hidden} {Targets} under {Imperfect} {Inspections},'' \emph{IFAC Proceedings
  Volumes}, vol.~46, no.~9, pp. 1656--1661, Jan. 2013. [Online]. Available:
  \url{https://www.sciencedirect.com/science/article/pii/S1474667016345311}
\BIBentrySTDinterwordspacing

\bibitem{laperriere-robillard_supervised_2022}
\BIBentryALTinterwordspacing
T.~Laperrière-Robillard, M.~Morin, and I.~Abi-Zeid, ``Supervised learning for
  maritime search operations: {An} artificial intelligence approach to search
  efficiency evaluation,'' \emph{Expert Systems with Applications}, vol. 206,
  p. 117857, Nov. 2022. [Online]. Available:
  \url{https://www.sciencedirect.com/science/article/pii/S0957417422011125}
\BIBentrySTDinterwordspacing

\bibitem{song_discrete_2004}
\BIBentryALTinterwordspacing
N.-O. Song and D.~Teneketzis, ``\BIBforeignlanguage{en}{Discrete search with
  multiple sensors},'' \emph{\BIBforeignlanguage{en}{Mathematical Methods of
  Operations Research}}, vol.~60, no.~1, pp. 1--13, Sep. 2004. [Online].
  Available: \url{https://doi.org/10.1007/s001860400360}
\BIBentrySTDinterwordspacing

\bibitem{koutsoupias_searching_1996}
E.~Koutsoupias, C.~H. Papadimitriou, and M.~Yannakakis, ``Searching a {Fixed}
  {Graph},'' in \emph{Proceedings of the 23rd {International} {Colloquium} on
  {Automata}, {Languages} and {Programming}}, ser. {ICALP} '96.\hskip 1em plus
  0.5em minus 0.4em\relax Berlin, Heidelberg: Springer-Verlag, Jul. 1996, pp.
  280--289.

\bibitem{ausiello_salesmen_2000}
G.~Ausiello, S.~Leonardi, and A.~Marchetti-Spaccamela,
  ``\BIBforeignlanguage{en}{On {Salesmen}, {Repairmen}, {Spiders}, and {Other}
  {Traveling} {Agents}},'' in \emph{\BIBforeignlanguage{en}{Algorithms and
  {Complexity}}}, G.~Bongiovanni, R.~Petreschi, and G.~Gambosi, Eds.\hskip 1em
  plus 0.5em minus 0.4em\relax Berlin, Heidelberg: Springer, 2000, pp. 1--16.

\bibitem{lau_optimal_2005}
\BIBentryALTinterwordspacing
H.~Lau, S.~Huang, and G.~Dissanayake, ``Optimal search for multiple targets in
  a built environment,'' in \emph{2005 {IEEE}/{RSJ} {International}
  {Conference} on {Intelligent} {Robots} and {Systems}}, Aug. 2005, pp.
  3740--3745, iSSN: 2153-0866. [Online]. Available:
  \url{https://ieeexplore.ieee.org/abstract/document/1544986/authors#authors}
\BIBentrySTDinterwordspacing

\bibitem{van_ee_approximation_2018}
\BIBentryALTinterwordspacing
M.~van Ee and R.~Sitters, ``Approximation and complexity of multi-target graph
  search and the {Canadian} traveler problem,'' \emph{Theoretical Computer
  Science}, vol. 732, pp. 14--25, Jul. 2018. [Online]. Available:
  \url{https://www.sciencedirect.com/science/article/pii/S0304397518302445}
\BIBentrySTDinterwordspacing

\bibitem{golden_orienteering_1987}
\BIBentryALTinterwordspacing
B.~L. Golden, L.~Levy, and R.~Vohra, ``\BIBforeignlanguage{en}{The orienteering
  problem},'' \emph{\BIBforeignlanguage{en}{Naval Research Logistics (NRL)}},
  vol.~34, no.~3, pp. 307--318, 1987, \_eprint:
  https://onlinelibrary.wiley.com/doi/pdf/10.1002/1520-6750\%28198706\%2934\%3A3\%3C307\%3A\%3AAID-NAV3220340302\%3E3.0.CO\%3B2-D.
  [Online]. Available:
  \url{https://onlinelibrary.wiley.com/doi/abs/10.1002/1520-6750%28198706%2934%3A3%3C307%3A%3AAID-NAV3220340302%3E3.0.CO%3B2-D}
\BIBentrySTDinterwordspacing

\bibitem{mohamed_person_2020}
\BIBentryALTinterwordspacing
S.~C. Mohamed, S.~Rajaratnam, S.~T. Hong, and G.~Nejat, ``Person {Finding}:
  {An} {Autonomous} {Robot} {Search} {Method} for {Finding} {Multiple}
  {Dynamic} {Users} in {Human}-{Centered} {Environments},'' \emph{IEEE
  Transactions on Automation Science and Engineering}, vol.~17, no.~1, pp.
  433--449, Jan. 2020, conference Name: IEEE Transactions on Automation Science
  and Engineering. [Online]. Available:
  \url{https://ieeexplore.ieee.org/document/8790973}
\BIBentrySTDinterwordspacing

\bibitem{karp_reducibility_1972}
\BIBentryALTinterwordspacing
R.~M. Karp, ``\BIBforeignlanguage{en}{Reducibility among {Combinatorial}
  {Problems}},'' in \emph{\BIBforeignlanguage{en}{Complexity of {Computer}
  {Computations}: {Proceedings} of a symposium on the {Complexity} of
  {Computer} {Computations}, held {March} 20–22, 1972, at the {IBM} {Thomas}
  {J}. {Watson} {Research} {Center}, {Yorktown} {Heights}, {New} {York}, and
  sponsored by the {Office} of {Naval} {Research}, {Mathematics} {Program},
  {IBM} {World} {Trade} {Corporation}, and the {IBM} {Research} {Mathematical}
  {Sciences} {Department}}}, R.~E. Miller, J.~W. Thatcher, and J.~D. Bohlinger,
  Eds.\hskip 1em plus 0.5em minus 0.4em\relax Boston, MA: Springer US, 1972,
  pp. 85--103. [Online]. Available:
  \url{https://doi.org/10.1007/978-1-4684-2001-2_9}
\BIBentrySTDinterwordspacing

\bibitem{Brennan_2015}
\BIBentryALTinterwordspacing
M.~Brennan, ``The hamptons’ billionaire lane: Where wall street’s richest
  retreat for the summer,'' May 2015. [Online]. Available:
  \url{https://www.forbes.com/sites/morganbrennan/2013/05/22/the-hamptons-billionaire-lane-where-wall-streets-richest-retreat-for-the-summer/}
\BIBentrySTDinterwordspacing

\bibitem{morin_ant_2023}
\BIBentryALTinterwordspacing
M.~Morin, I.~Abi-Zeid, and C.-G. Quimper, ``Ant colony optimization for path
  planning in search and rescue operations,'' \emph{European Journal of
  Operational Research}, vol. 305, no.~1, pp. 53--63, Feb. 2023. [Online].
  Available:
  \url{https://www.sciencedirect.com/science/article/pii/S0377221722004945}
\BIBentrySTDinterwordspacing

\bibitem{arora_polynomial_1998}
\BIBentryALTinterwordspacing
S.~Arora, ``Polynomial time approximation schemes for {Euclidean} traveling
  salesman and other geometric problems,'' \emph{Journal of the ACM}, vol.~45,
  no.~5, pp. 753--782, Sep. 1998. [Online]. Available:
  \url{https://dl.acm.org/doi/10.1145/290179.290180}
\BIBentrySTDinterwordspacing

\bibitem{mitchell_guillotine_1999}
\BIBentryALTinterwordspacing
J.~S.~B. Mitchell, ``Guillotine {Subdivisions} {Approximate} {Polygonal}
  {Subdivisions}: {A} {Simple} {Polynomial}-{Time} {Approximation} {Scheme} for
  {Geometric} {TSP}, k-{MST}, and {Related} {Problems},'' \emph{SIAM Journal on
  Computing}, vol.~28, no.~4, pp. 1298--1309, Jan. 1999, publisher: Society for
  Industrial and Applied Mathematics. [Online]. Available:
  \url{https://epubs.siam.org/doi/abs/10.1137/S0097539796309764}
\BIBentrySTDinterwordspacing

\bibitem{gunawan2016orienteering}
A.~Gunawan, H.~C. Lau, and P.~Vansteenwegen, ``Orienteering problem: A survey
  of recent variants, solution approaches and applications,'' \emph{European
  Journal of Operational Research}, vol. 255, no.~2, pp. 315--332, 2016.

\bibitem{tsiligirides1984heuristic}
T.~Tsiligirides, ``Heuristic methods applied to orienteering,'' \emph{Journal
  of the Operational Research Society}, vol.~35, no.~9, pp. 797--809, 1984.

\bibitem{chao1996fast}
I.-M. Chao, B.~L. Golden, and E.~A. Wasil, ``A fast and effective heuristic for
  the orienteering problem,'' \emph{European journal of operational research},
  vol.~88, no.~3, pp. 475--489, 1996.

\bibitem{christofides_worst-case_2022}
\BIBentryALTinterwordspacing
N.~Christofides, ``\BIBforeignlanguage{en}{Worst-{Case} {Analysis} of a {New}
  {Heuristic} for the {Travelling} {Salesman} {Problem}},''
  \emph{\BIBforeignlanguage{en}{Operations Research Forum}}, vol.~3, no.~1,
  p.~20, Mar. 2022. [Online]. Available:
  \url{https://doi.org/10.1007/s43069-021-00101-z}
\BIBentrySTDinterwordspacing

\bibitem{barkaoui_information-theoretic-based_2014}
\BIBentryALTinterwordspacing
M.~Barkaoui, J.~Berger, and A.~Boukhtouta, ``An information-theoretic-based
  evolutionary approach for the dynamic search path planning problem,'' in
  \emph{2014 {International} {Conference} on {Advanced} {Logistics} and
  {Transport} ({ICALT})}, May 2014, pp. 7--12. [Online]. Available:
  \url{https://ieeexplore.ieee.org/document/6864073}
\BIBentrySTDinterwordspacing

\bibitem{berger_information_2015}
\BIBentryALTinterwordspacing
J.~Berger, N.~Lo, A.~Boukhtouta, and M.~Noel, ``\BIBforeignlanguage{en}{An
  information theoretic based integer linear programming approach for the
  discrete search path planning problem},''
  \emph{\BIBforeignlanguage{en}{Optimization Letters}}, vol.~9, no.~8, pp.
  1585--1607, Dec. 2015. [Online]. Available:
  \url{https://doi.org/10.1007/s11590-015-0874-7}
\BIBentrySTDinterwordspacing

\bibitem{teller_minimizing_2019}
\BIBentryALTinterwordspacing
R.~Teller, M.~Zofi, and M.~Kaspi, ``\BIBforeignlanguage{en}{Minimizing the
  average searching time for an object within a graph},''
  \emph{\BIBforeignlanguage{en}{Computational Optimization and Applications}},
  vol.~74, no.~2, pp. 517--545, Nov. 2019. [Online]. Available:
  \url{https://doi.org/10.1007/s10589-019-00121-w}
\BIBentrySTDinterwordspacing

\bibitem{barkaoui_evolutionary_2019}
\BIBentryALTinterwordspacing
M.~Barkaoui, J.~Berger, and A.~Boukhtouta, ``\BIBforeignlanguage{en}{An
  evolutionary approach for the target search problem in uncertain
  environment},'' \emph{\BIBforeignlanguage{en}{Journal of Combinatorial
  Optimization}}, vol.~38, no.~3, pp. 808--835, Oct. 2019. [Online]. Available:
  \url{https://doi.org/10.1007/s10878-019-00413-1}
\BIBentrySTDinterwordspacing

\bibitem{senanayake_search_2016}
\BIBentryALTinterwordspacing
M.~Senanayake, I.~Senthooran, J.~C. Barca, H.~Chung, J.~Kamruzzaman, and
  M.~Murshed, ``Search and tracking algorithms for swarms of robots: {A}
  survey,'' \emph{Robotics and Autonomous Systems}, vol.~75, pp. 422--434, Jan.
  2016. [Online]. Available:
  \url{https://www.sciencedirect.com/science/article/pii/S0921889015001876}
\BIBentrySTDinterwordspacing

\bibitem{wong_multi-vehicle_2005}
\BIBentryALTinterwordspacing
E.-M. Wong, F.~Bourgault, and T.~Furukawa, ``Multi-vehicle {Bayesian} {Search}
  for {Multiple} {Lost} {Targets},'' in \emph{Proceedings of the 2005 {IEEE}
  {International} {Conference} on {Robotics} and {Automation}}, Apr. 2005, pp.
  3169--3174, iSSN: 1050-4729. [Online]. Available:
  \url{https://ieeexplore.ieee.org/document/1570598}
\BIBentrySTDinterwordspacing

\bibitem{lo_toward_2012}
\BIBentryALTinterwordspacing
N.~Lo, J.~Berger, and M.~Noel, ``Toward optimizing static target search path
  planning,'' in \emph{2012 {IEEE} {Symposium} on {Computational}
  {Intelligence} for {Security} and {Defence} {Applications}}, Jul. 2012, pp.
  1--7, iSSN: 2329-6275. [Online]. Available:
  \url{https://ieeexplore.ieee.org/document/6291538}
\BIBentrySTDinterwordspacing

\bibitem{wettergren_discrete_2014}
\BIBentryALTinterwordspacing
T.~A. Wettergren and J.~G. Baylog, ``Discrete search allocation with object
  uncertainty,'' \emph{International Journal of Operational Research}, vol.~20,
  no.~1, pp. 1--20, Jan. 2014, publisher: Inderscience Publishers. [Online].
  Available:
  \url{https://www.inderscienceonline.com/doi/abs/10.1504/IJOR.2014.060513}
\BIBentrySTDinterwordspacing

\bibitem{dell_using_1996}
\BIBentryALTinterwordspacing
R.~Dell, J.~N. Eagle, G.~H.~A. Martins, and A.~G. Santos,
  ``\BIBforeignlanguage{en}{Using multiple searchers in constrained‐path,
  moving‐target search problems},'' \emph{\BIBforeignlanguage{en}{Naval
  Research Logistics}}, 1996. [Online]. Available:
  \url{https://www.semanticscholar.org/paper/Using-multiple-searchers-in-constrained%E2%80%90path%2C-Dell-Eagle/b8079653505abf9120cca0f0ad5df8b945f837a4}
\BIBentrySTDinterwordspacing

\bibitem{lau_discounted_2008}
\BIBentryALTinterwordspacing
H.~Lau, S.~Huang, and G.~Dissanayake, ``Discounted {MEAN} bound for the optimal
  searcher path problem with non-uniform travel times,'' \emph{European Journal
  of Operational Research}, vol. 190, no.~2, pp. 383--397, Oct. 2008. [Online].
  Available:
  \url{https://www.sciencedirect.com/science/article/pii/S0377221707006315}
\BIBentrySTDinterwordspacing

\bibitem{raap_moving_2019}
\BIBentryALTinterwordspacing
M.~Raap, M.~Preuß, and S.~Meyer-Nieberg, ``Moving target search optimization
  – {A} literature review,'' \emph{Computers \& Operations Research}, vol.
  105, pp. 132--140, May 2019. [Online]. Available:
  \url{https://www.sciencedirect.com/science/article/pii/S0305054819300127}
\BIBentrySTDinterwordspacing

\bibitem{delavernhe_planning_2021}
\BIBentryALTinterwordspacing
F.~Delavernhe, P.~Jaillet, A.~Rossi, and M.~Sevaux, ``Planning a multi-sensors
  search for a moving target considering traveling costs,'' \emph{European
  Journal of Operational Research}, vol. 292, no.~2, pp. 469--482, Jul. 2021.
  [Online]. Available:
  \url{https://www.sciencedirect.com/science/article/pii/S0377221720309589}
\BIBentrySTDinterwordspacing

\bibitem{dasgupta_aggregation-based_2004}
\BIBentryALTinterwordspacing
B.~Dasgupta, J.~Hespanha, and E.~Sontag, ``Aggregation-based approaches to
  honey-pot searching with local sensory information,'' in \emph{Proceedings of
  the 2004 {American} {Control} {Conference}}, vol.~2, Jun. 2004, pp.
  1202--1207 vol.2, iSSN: 0743-1619. [Online]. Available:
  \url{https://ieeexplore.ieee.org/document/1386736}
\BIBentrySTDinterwordspacing

\bibitem{meghjani_multi-target_2016}
\BIBentryALTinterwordspacing
M.~Meghjani, S.~Manjanna, and G.~Dudek, ``Multi-target search strategies,'' in
  \emph{2016 {IEEE} {International} {Symposium} on {Safety}, {Security}, and
  {Rescue} {Robotics} ({SSRR})}, Oct. 2016, pp. 328--333. [Online]. Available:
  \url{https://ieeexplore.ieee.org/document/7784323}
\BIBentrySTDinterwordspacing

\bibitem{huynh_et_al:LIPIcs.SWAT.2024.27}
\BIBentryALTinterwordspacing
K.~C. Huynh, J.~S.~B. Mitchell, L.~Nguyen, and V.~Polishchuk, ``{Optimizing
  Visibility-Based Search in Polygonal Domains},'' in \emph{19th Scandinavian
  Symposium and Workshops on Algorithm Theory (SWAT 2024)}, ser. Leibniz
  International Proceedings in Informatics (LIPIcs), H.~L. Bodlaender, Ed.,
  vol. 294.\hskip 1em plus 0.5em minus 0.4em\relax Dagstuhl, Germany: Schloss
  Dagstuhl -- Leibniz-Zentrum f{\"u}r Informatik, 2024, pp. 27:1--27:16.
  [Online]. Available:
  \url{https://drops.dagstuhl.de/entities/document/10.4230/LIPIcs.SWAT.2024.27}
\BIBentrySTDinterwordspacing

\bibitem{sato_path_2010}
\BIBentryALTinterwordspacing
H.~Sato and J.~O. Royset, ``\BIBforeignlanguage{en}{Path optimization for the
  resource-constrained searcher},'' \emph{\BIBforeignlanguage{en}{Naval
  Research Logistics (NRL)}}, vol.~57, no.~5, pp. 422--440, 2010, \_eprint:
  https://onlinelibrary.wiley.com/doi/pdf/10.1002/nav.20411. [Online].
  Available: \url{https://onlinelibrary.wiley.com/doi/abs/10.1002/nav.20411}
\BIBentrySTDinterwordspacing

\bibitem{morin_hybrid_2013}
\BIBentryALTinterwordspacing
M.~Morin, I.~Abi-Zeid, Y.~Petillot, and C.-G. Quimper, ``A hybrid algorithm for
  coverage path planning with imperfect sensors,'' in \emph{2013 {IEEE}/{RSJ}
  {International} {Conference} on {Intelligent} {Robots} and {Systems}}, Nov.
  2013, pp. 5988--5993, iSSN: 2153-0866. [Online]. Available:
  \url{https://ieeexplore.ieee.org/document/6697225}
\BIBentrySTDinterwordspacing

\bibitem{banerjee_multi-agent_2023}
\BIBentryALTinterwordspacing
A.~Banerjee, R.~Ghods, and J.~Schneider, ``Multi-{Agent} {Active} {Search}
  using {Detection} and {Location} {Uncertainty},'' \emph{2023 IEEE
  International Conference on Robotics and Automation (ICRA)}, pp. 7720--7727,
  May 2023, conference Name: 2023 IEEE International Conference on Robotics and
  Automation (ICRA) ISBN: 9798350323658 Place: London, United Kingdom
  Publisher: IEEE. [Online]. Available:
  \url{https://ieeexplore.ieee.org/document/10161017/}
\BIBentrySTDinterwordspacing

\bibitem{collins_probabilistically_2024}
\BIBentryALTinterwordspacing
M.~Collins, J.~J. Beard, N.~Ohi, and Y.~Gu, ``Probabilistically {Informed}
  {Robot} {Object} {Search} with {Multiple} {Regions},'' Apr. 2024,
  arXiv:2404.04186 [cs] version: 1. [Online]. Available:
  \url{http://arxiv.org/abs/2404.04186}
\BIBentrySTDinterwordspacing

\end{thebibliography}

\end{document}